%% file: main_camera_ready.tex
\documentclass[final,12pt]{colt2025} 

\input{macros}
\usepackage{minitoc}

\newcommand{\abs}[1]{\left|#1\right|}

\newcommand{\norm}[1]{\left\|#1\right\|}

\def\x{\bm{x}}

\newtoggle{longversion}
\settoggle{longversion}{true}

\title[A Gap Between the Gaussian RKHS and Neural Networks]{A Gap Between the Gaussian RKHS and Neural Networks: \\ An Infinite-Center Asymptotic Analysis}
\usepackage{times}

\coltauthor{%
 \Name{Akash Kumar} \Email{akk002@ucsd.edu}\\
 \addr Department of Computer Science and Engineering \\
 \addr University of California, San Diego
 \AND
 \Name{Rahul Parhi} \Email{rahul@ucsd.edu}\\
 \addr Department of Electrical and Computer Engineering \\
 \addr University of California, San Diego
 \AND
 \Name{Mikhail Belkin} \Email{mbelkin@ucsd.edu}\\
 \addr Department of Computer Science and Engineering \\
 \addr Halıcıoğlu Data Science Institute \\
 \addr University of California, San Diego
}

\begin{document}

\maketitle
\begin{abstract}

Recent works have characterized the function-space inductive bias of infinite-width bounded-norm single-hidden-layer neural networks as a kind of bounded-variation-type space. This novel neural network Banach space encompasses many classical multivariate function spaces, including certain Sobolev spaces and the spectral Barron spaces. Notably, this Banach space also includes functions that exhibit less classical regularity, such as those that only vary in a few directions. On bounded domains, it is well-established that the Gaussian reproducing kernel Hilbert space (RKHS) strictly embeds into this Banach space, demonstrating a clear gap between the Gaussian RKHS and the neural network Banach space. It turns out that when investigating these spaces on unbounded domains, e.g., all of $\mathbb{R}^d$, the story is fundamentally different. We establish the following fundamental result: Certain functions that lie in the Gaussian RKHS have infinite norm in the neural network Banach space. This provides a nontrivial gap between kernel methods and neural networks by exhibiting functions that kernel methods easily represent, whereas neural networks cannot.
\end{abstract}

\input{intro}

\input{related_work}
\input{setup}
\input{R-norm}

\input{infinite_RKHS}

\section{Conclusion}
In this work, we showed that if we allow \tt{unbounded} domains, there exist functions that cannot be represented within $\rbv{\reals^d}$, but can be represented within the Gaussian RKHS. This analysis reveals a nontrivial gap between kernel methods and neural networks by exhibiting functions that kernel methods can represent, whereas neural networks cannot.
%
On that note, this observation motivates further investigation of what this gap entails in a learning setting. This observation also motivates investigating if a similar gap exists between kernel methods and \tt{deep} neural networks. We leave the details to future work.

\section*{Acknowledgement}
Authors thank anonymous reviewers for helpful feedback on the work. AK thanks the National Science Foundation for support under grant IIS-2211386 in the duration of this project. MB acknowledges support from the National Science Foundation (NSF) and the Simons Foundation for the Collaboration on the Theoretical Foundations of Deep Learning through awards DMS-2031883 and \#814639 as well as the  TILOS institute (NSF CCF-2112665) and the Office of Naval Research (ONR N000142412631).

\bibliography{ref}
 \iftoggle{longversion}{

\setcounter{tocdepth}{-5}  
\addtocontents{toc}{\protect\setcounter{tocdepth}{2}}  

\renewcommand{\contentsname}{A Gap Between the Gaussian RKHS and Neural Networks:
Supplementary Materials}
\tableofcontents

\appendix
\input{app_explicit}
\input{app_change_of_variable}
\input{app_usefulproperty}
\input{app_example}

}
{}
\end{document}

%% file: macros.tex
\usepackage{savesym,url,graphicx,hyperref,color}
\usepackage[titletoc]{appendix}
\usepackage{subcaption}
\usepackage{multicol}
\usepackage{blkarray}
\usepackage{wrapfig}

\usepackage{etoolbox}
\usepackage{hyperref}       
\usepackage{url}            
\usepackage{booktabs}       
\usepackage{tikz}
\usepackage{pgfplots}
\usepackage{verbatim}
\usepackage{amsmath,bm}
\usepackage{multirow}
\usepackage{mathtools}
\usepackage{comment}

\newtoggle{showcomments}
\settoggle{showcomments}{true}

\iftoggle{showcomments}{
\newcommand{\akash}[1]{{\textcolor{purple}{[{\bf A:} #1]}}}
\newcommand{\rahul}[1]{{\textcolor{magenta}{[{\bf R:} #1]}}}
\newcommand{\parthe}[1]{{\textcolor{red}{[{\bf A:} #1]}}}
\newcommand{\annotate}[1]{{\textcolor{blue}{[{\bf Note:} #1]}}}
}
{
\newcommand{\akash}[1]{}
\newcommand{\rahul}[1]{}
\newcommand{\parthe}[1]{}
\newcommand{\annotate}[1]{}
}

\newtheorem{assumption}{Assumption}

\makeatletter

\newcommand{\Rmnum}[1]{\expandafter\@slowromancap\romannumeral #1@}
\makeatother
\usepackage{caption}

\newcommand{\exmref}[1]{Example~\ref{#1}}

\newcommand{\figref}[1]{Fig.~\ref{#1}}
\newcommand{\eqnref}[1]{\text{Eq.}~(\ref{#1})}
\newcommand{\secref}[1]{\textnormal{Section}~\ref{#1}}
\newcommand{\appref}[1]{Appendix \ref{#1}}

\newcommand{\thmref}[1]{Theorem~\ref{#1}}

\newcommand{\lemref}[1]{Lemma~\ref{#1}}

\newcommand{\assref}[1]{Assumption~\ref{#1}}

\newcommand{\rtv}[1]{\mathcal{R}\mathrm{TV}^2({#1})}
\newcommand{\rbv}[1]{\mathcal{R}\mathrm{BV}^2({#1})}

\newcommand{\paren} [1] {\ensuremath{ \left( {#1} \right) }}

\newcommand{\bracket}[1]{\left[#1\right]}

\newcommand{\curlybracket}[1]{\ensuremath{\left\{#1\right\}}}

\newcommand{\vol}[1]{\sf{vol}({#1})}
\newcommand{\symmp}{\text{Sym}_{+}(\reals^{d\times d})}

\newcommand{\reals}{\ensuremath{\mathbb{R}}}

\newcommand{\cR}{{\mathcal{R}}}

\newcommand{\cD}{{\mathcal{D}}}

\newcommand{\cK}{{\mathcal{K}}}
\newcommand{\cX}{{\mathcal{X}}}

\newcommand{\cH}{{\mathcal{H}}}

\newcommand{\cC}{{\mathcal{C}}}

\newcommand{\bL}{{\mathbf{L}}}
\newcommand{\bM}{{\mathbf{M}}}

\renewcommand{\tt}[1]{\textit{#1}}
\renewcommand{\sf}[1]{\textsf{#1}}
\def\BState{\State\hskip-\ALG@thistlm}

\DeclareMathOperator*{\esssup}{ess\,sup}

%% file: intro.tex
\section{Introduction}
In supervised learning, we observe samples with corresponding labels, which may represent classes or continuous values. Our primary objective is to construct a function \(f: \mathbb{R}^d \to \mathbb{R}\) based on these observations that can accurately predict labels for new, unseen data points. Traditionally, reproducing kernel Hilbert spaces (RKHS) have provided a principled framework for this task, offering both theoretical guarantees and practical algorithms. Their power stems from the representer theorem, which ensures that optimal solutions can be expressed as combinations of kernel functions centered at the training points.

However, the landscape of machine learning has evolved significantly with the emergence of neural networks, which have demonstrated remarkable success across diverse applications over kernel methods. The simplest neural architecture—the single-hidden layer network—builds upon the concept of ridge functions, which map \(\mathbb{R}^d \to \mathbb{R}\) via the form $\bm{x} \mapsto \sigma(\bm{w}^\top \bm{x})$, where \(\sigma: \mathbb{R} \to \mathbb{R}\) is a univariate function and \(\bm{w} \in \mathbb{R}^d \setminus \{\bm{0}\}\). In practice, these networks combine multiple ridge functions:
\begin{equation}
\bm{x} \mapsto \sum_{k=1}^K v_k\, \sigma(\bm{w}_k^\mathsf{T} \bm{x} - b_k),
\end{equation}
where \(K\) represents the network width, \(v_k \in \mathbb{R}\) and \(\bm{w}_k \in \mathbb{R}^d \setminus \{\bm{0}\}\) are weights, and \(b_k \in \mathbb{R}\) are biases. While RKHS methods suffer from the curse of dimensionality, neural networks can overcome it by learning effective low-dimensional representations~\citep{ghorbani_curse, luxberg_curse}. 

A fundamental question is to compare the approximation capabilities of neural networks with those of RKHS corresponding to different kernels. For example, \cite{mei_rate} showed that if the target function is a single neuron, neural networks can learn efficiently using roughly \(d \log d\) samples, whereas the corresponding RKHS requires a sample size that grows polynomially in the dimension \(d\) (see also~\citet{yehudai2019power, ghorbani2019limitations}).

Recent work~\citep{Parhi2020BanachSR,near_mini} has studied the Banach-space optimality of single-hidden-layer (shallow ReLU) networks over both bounded and unbounded domains \(\Omega \subseteq \reals^d\).  There, the authors established a representer theorem which demonstrates that solutions to data-fitting problems in these networks naturally reside in a kind of bounded variation space, referred to as the second-order Radon bounded variation space \(\rbv{\Omega}\). These spaces, in turn, contain several classical multivariate function spaces, including certain Sobolev spaces as well as certain spectral Barron spaces~\citep{barron_universal}. For instance, \citet{near_mini} have shown that the Sobolev space \(H^{d+1}(\Omega)\) embeds into \(\rbv{\Omega}\) for any bounded Lipschitz domain $\Omega \subset \reals^d$. Moreover, on any bounded Lipschitz domain \(\Omega \subset \reals^d\), the Gaussian reproducing kernel Hilbert space \(\cH^{\sf{Gauss}}(\Omega)\) is known to embed into the Sobolev space \(H^{s}(\Omega)\) for all \(s> 0\) (see Corollary 4.36 of~\citet{steinwart_svm}). This observation appears to highlight limitations of Gaussian kernel machines when compared to neural networks on \emph{bounded domains}. Consequently, a natural question arises.
\begin{center}
    \itshape Are Gaussian kernel machines restrictive in approximating general functions?
\end{center}
Conversely, one may also ask the following question.
\begin{center}
    \itshape To what extent can we approximate functions using shallow neural networks?
\end{center}

To that end, \cite{Ghorbani_2021} demonstrated that the gap between neural network approximations and kernel methods can be narrowed when the intrinsic dimensionality of the target function is well captured by the covariates of the data. 
In this paper, we take a different perspective: While the Gaussian RKHS may seem rather limited in a bounded domain, we show that on unbounded domains, in particular, on $\reals^d$ with fixed dimension \(d\), there exist functions in \(\cH^{\sf{Gauss}}(\reals^d)\) with unbounded $\rbv{\reals^d}$-norm. 

The key idea behind our analysis is that, in the regime of kernel machines with infinite centers on $\reals^d$, there exist functions of the form
$f = \sum_{i=1}^\infty \alpha_i k(\bm{x}_i,\cdot)$ with bounded RKHS norm, but the infinite sequence $\{\alpha_i\}$ has an \tt{unbounded} $\ell_1$-norm (see \exmref{exam: divergence} in \secref{sec: setup}). This fact can be exploited to design a sequence of functions $\curlybracket{f_n}$ whose $\rbv{\reals^d}$-norm is diverging as $n \to \infty$ (see \thmref{thm: unbounded} in \secref{sec: diverge}). An important step in this study is we compute an explicit form for $\rbv{\reals^d}$-norm of a Gaussian kernel machine, and further simplify the form using well-known Hermite polynomials. This form provides an interpretable characterization of these kernel machines, which is of independent interest for future studies.

%% file: related_work.tex
\section{Related Work}

\paragraph{Approximability with Kernel Methods}

\cite{bach_curse} studied various classes of single-/multi-index models with low intrinsic dimension and bounded $\rbv{\reals^d}$-norm. In contrast, \cite{ghorbani2019limitations} showed that if the covariates have the same dimension as the low intrinsic dimension of the target function, kernel and neural network approximations can be competitive. Empirically, some works show that the curse of dimensionality with kernel methods can be handled with an appropriate choice of dataset-specific kernels~\citep{arora2019harnessing, novak2018bayesian, shankar2020neural} or mirroring neural network training dynamics closely to kernel methods~\citep{mean_field_songmei, sirignano2020mean, rotskoff2022trainability, chizat2018global}. Furthermore, \citet{Petrini_2023} showed that compared to a network that learns sparse representations while the target function is constant or smooth along certain directions of the input space, lazy training (via random feature kernel or the NTK) yields better performance. 
But a wide body of work has also shown a gap in approximation with neural networks capturing a richer and more nuanced class of functions compared to kernel methods (see ~\citep{allen2019can,mean_field_songmei, yehudai2019power, ghorbani2019limitations}). In our work, we show that while Gaussian RKHS is embedded within neural networks in bounded domains, in the unbounded regime there exists a non-trivial gap between $\cH^{\sf{Gauss}}(\reals^d)$ and $\rbv{\reals^d}$.

\paragraph{Function Spaces of Shallow Networks} The function space \(\rbv{\Omega}\) naturally characterizes the function approximation and representation capabilities of shallow ReLU neural networks~\citep{Ongie2020A}. \citet{Parhi2020BanachSR} established a \textit{representer theorem}, showing that solutions to variational problems over \(\rbv{\Omega}\) correspond to single-hidden layer ReLU networks with weight decay regularization. Unlike RKHSs
\(\rbv{\Omega}\) can efficiently represent functions with a low-dimensional structure. 
Moreover, neural networks trained with weight decay achieve near-minimax optimal estimation rates for functions in $\rbv{\Omega}$, while kernel methods provably cannot \citep{near_mini}. This suggests that on bounded domains, RKHSs are quite restrictive, while \(\rbv{\Omega}\) provides a more expressive framework.
For further details see \citep{Ongie2020A,Parhi2020BanachSR,Parhi2021WhatKO,near_mini,parhi2023deep,BARTOLUCCI2023194,Parhi2023FunctionSpaceOO}

\paragraph{Embeddings of RKHSs and $\rbv{\Omega}$} 
For any bounded Lipshitz domain $\Omega \subseteq \reals^d$, it is well-known that the Sobolev space \( H^s(\Omega) \) is (equivalent to) an RKHS if and only if \( s > d/2 \). For example, the Laplace and Matérn kernels are associated with Sobolev RKHSs \citep[see, e.g.,][Example~2.6]{kanagawa2018gaussianprocesseskernelmethods}. In contrast, \citet{zhou_embed} and \citet[cf.,][Corollary 4.36]{steinwart_svm} showed that the Gaussian RKHS \( \mathcal{H}^{\text{Gauss}}(\Omega) \) is contained in \( \mathcal{H}^{\text{Gauss}}(\Omega) \subset H^s(\Omega) \) for all \( s \geq 0 \). 
Recent work has further demonstrated that the RKHSs of typical neural tangent kernel (NTK) and neural network Gaussian process (NNGP) kernels for the ReLU activation function are equivalent to the Sobolev spaces \( H^{(d+1)/2}(\mathbb{S}^d) \) and \( H^{(d+3)/2}(\mathbb{S}^d) \), respectively~\citep{bietti2021deep,chen2021deep}. \cite{Steinwart2009OptimalRF} has shown that an optimal learning rates in Sobolev RKHSs can be achieved by cross-validating the regularization parameter. On another front, embedding properties relating Sobolev spaces and the second-order Radond-domain bounded variation space has been explored. For example, \cite{Ongie2020A} showed that \( W^{d+1}(L_1(\mathbb{R}^d)) \) embeds in \( \rbv{\mathbb{R}^d} \). More recently, \cite{mao2024approximationratesshallowreluk} established a sharp bound by proving that \( W^s(L_p(\Omega)) \) with \( s \geq 2 + (d+1)/2 \) for $p \ge 2$ embeds in \( \rbv{\Omega} \) for bounded domains \( \Omega \subset \mathbb{R}^d \). 

%% file: setup.tex
\section{Problem Setup and Preliminaries}\label{sec: setup}


\subsection{Gaussian Reproducing Kernel Hilbert Space}
We begin by defining a reproducing kernel Hilbert space (RKHS) associated with a Gaussian kernel on an infinite domain. For a given positive definite Mahalanobis matrix $\mathbf{M} \in \text{Sym}_{+}(\mathbb{R}^{d \times d})$, we define the Gaussian kernel $k_{\mathbf{M}}: \mathbb{R}^d \times \mathbb{R}^d \to \mathbb{R}$ as
\begin{equation}
    k_{\mathbf{M}}(\bm{x},\bm{y}) = \exp\left(-\frac{\|\bm{x}-\bm{y}\|_\mathbf{M}^2}{2\sigma^2}\right),
\end{equation}
where $\sigma > 0$ is a fixed scale parameter and the Mahalanobis distance is defined as
\begin{equation}
    \|\bm{x} - \bm{y}\|_\mathbf{M}^2 = (\bm{x} - \bm{y})^{\mathsf{T}}\mathbf{M}(\bm{x} - \bm{y}).
\end{equation}
The corresponding RKHS $\mathcal{H}$ is defined as the closure\footnote{With respect to the norm topology on $\cH$.} of the linear span of kernel functions:
\begin{align}
    \cH := \sf{cl}\paren{\curlybracket{f: \cX \to \reals\,\bigg\lvert\, n \in \mathbb{N}, f(\cdot) = \sum_{i = 1}^n \alpha_i \cdot k_{\mathbf{M}}(\bm{x}_i,\cdot),\, \bm{x}_i \in \reals^d}}, \label{eq: infinitegauss}
\end{align}
where the (squared) RKHS norm $\|\cdot\|_{\cH}^2$ of a kernel machine $f \in \cH$ is defined as $\|f\|_{\cH}^2 = \sum_{i,j} \alpha_i \alpha_j k_{\bm{M}}(\bm{x}_i,\bm{x}_j)$. Alternately, we can write $\|f\|_{\cH}^2 = \alpha^{\mathsf{T}}\mathbf{K}\alpha$ where $\mathbf{K} = (k_{\mathbf{M}}(\bm{x}_i,\bm{x}_j))_{i,j}$ is an $n \times n$ matrix. 

\subsection{Separated Sets and Function Spaces}
For our analysis, we introduce two key definitions of separated sets that play a crucial role in our theoretical development.
\begin{definition}[$(\bm{\beta}, \delta)$-separated set]
For any given scalar $\delta > 0$ and a vector $\bm{\beta} \in \reals^d$, a $(\bm{\beta}, \delta)$-separated subset of size $n \in \mathbb{N}$ is defined as  
\begin{align}
  \cC_n(\bm{\beta}, \delta) := \curlybracket{\{\bm{x}_1,\ldots, \bm{x}_n\} \,\middle|\, \, \forall\, i,j,\, |\bm{\beta}^{\mathsf{T}} \bm{x}_i - \bm{\beta}^{\mathsf{T}} \bm{x}_j| \ge \delta }.  
\end{align}
\end{definition}
This could be further generalized to the notion
\begin{definition}[$(\bm{\beta}, \delta, \eta)$-separated set] For any given scalars $\delta, \eta > 0$ and a vector $\bm{\beta} \in \reals^d$ a $(\bm{\beta}, \delta, \eta)$-separated subset of size $n \in \mathbb{N}$ is defined as 
\begin{align}
    \cC_n(\bm{\beta}, \delta, \eta) := \curlybracket{\{\bm{x}_1,\ldots, \bm{x}_n\} \,\middle|\, \, \forall\, i,j, \bm{\beta'} \in \reals^d \text{ s.t. } \bm{\beta}^{\mathsf{T}}\bm{\beta}' \ge \eta \norm{\bm{\beta}}\norm{\bm{\beta}'}, |\bm{\beta}'^{\mathsf{T}} \bm{x}_i - \bm{\beta}'^{\mathsf{T}} \bm{x}_j| \ge \delta}.
\end{align}
\end{definition}

\begin{example}\label{exam: set}
Let $\bm{\beta} = (1,0,\dots,0)$. For all $\eta_0 \ge \eta$, pick $\bm{\beta}' 
  \;=\; 
  \bigl(\eta_0,\sqrt{1-\eta_0^2},0,\dots,0\bigr)$ 
  so that
  $\|\bm{\beta}'\|=1 
  \;\;\text{and}\;\; 
  \bm{\beta}^{\mathsf{T}} \bm{\beta}' \;=\; \eta_0 \ge \eta$. Now define 
  \begin{align}
  \bm{x}_i := (i-1)\,\delta\,\bm{\beta}', 
  \quad i=1,\dots,n.    
  \end{align}
For $i\neq j$, we have 
\begin{align}
  \bigl|\bm{\beta}'^{\mathsf{T}} \bm{x}_i 
        - \bm{\beta}'^{\mathsf{T}} \bm{x}_j\bigr|
  = |(i-j)|\,\delta\,\|\bm{\beta}'\|^2 
  = |i-j|\,\delta 
  \,\ge\, \delta.  
\end{align}
Hence, $\{\bm{x}_1,\dots,\bm{x}_n\}$ is in $(\bm{\beta},\delta,\eta)$-separated subset of size $n$.
\end{example}

\begin{definition}[Unbounded Combinations]
For a kernel machine $f \in \mathcal{H}$ with representation $f = \sum_{i=1}^\infty \alpha_i k(\bm{x}_i,\cdot)$, we say the coefficient vector $\alpha = (\alpha_i)_{i=1}^\infty$ is unbounded with respect to $f$ if $\|\alpha\|_{\ell_1}= \sum_{i=1}^\infty |\alpha_i| = \infty$.
\end{definition}

\begin{example}\label{exam: divergence}
Consider a kernel machine $f \in \cH$ corresponding to the combination $\alpha = (a_n)$ defined by $a_n = \frac{1}{n}$ for each $n \in \mathbb{N}$ and a sequence of centers $(\bm{x}_n) \subset \reals^d$ such that for all $i,j$, $\norm{\bm{x}_i - \bm{x}_j} \ge |i-j| \delta$ for some fixed scalar $\delta > 0$. For this construction, Gaussian RKHS norm $\alpha^{\mathsf{T}} \mathbf{K} \alpha = \sum_{i = 1}^\infty\sum_{j = 1}^\infty \alpha_i\alpha_j k(\bm{x}_i,\bm{x}_j) < \infty$, but $\|\alpha\|_{\ell_1}$ is unbounded. 
\end{example}
We provide the proof of the statement in the example above in \iftoggle{longversion}{\appref{app: diverging}}{the supplemental materials}.

\subsection{Probabilist's Hermite Polynomials}
Probabilist's Hermite polynomials~\citep{szegő1975orthogonal}, denoted by \( \text{He}_d(z) : \reals \to \reals \), are defined by the generating function
\begin{align}
\exp\Bigl(zt-\frac{t^2}{2}\Bigr)=\sum_{d=0}^\infty \text{He}_d(z)\frac{t^d}{d!},    
\end{align}
and they are orthogonal with respect to the standard normal density
\begin{align}
\int_{-\infty}^{\infty} \text{He}_d(z)\, \text{He}_{d'}(z)\, \frac{1}{\sqrt{2\pi}} \exp\!\Bigl(-\frac{z^2}{2}\Bigr)\,dz = d!\,\delta_{dd'},    
\end{align}
where \(\delta_{dd'}\) is the Kronecker delta. We use the notation $H_d$ to denote the polynomial unless stated otherwise.

\subsection{Radon Transform and the Second-Order Radon-Domain Bounded Variation Space}

\paragraph{Radon Transform} For a function \(f : \mathbb{R}^d \to \mathbb{R}\), its Radon transform 
\(\mathcal{R}\{f\}\) is defined by
\begin{align}
  \mathcal{R}\{f\}(\bm{\beta}, t)
  = \int_{\{ \bm{x} \in \mathbb{R}^d : \bm{\beta}^{\mathsf{T}} \bm{x} = t \}} 
    f(\bm{x})\,ds(\bm{x}),  
\end{align}
where \(\bm{u} \in \mathbb{S}^{d-1}\), \(t \in \mathbb{R}\), and \(ds(\bm{x})\) is the 
\((d-1)\)-dimensional Lebesgue measure on the hyperplane 

\paragraph{Radon Bounded Variation Space} We define the second-order Radon bounded variation space $\rbv{\reals^d}$ as:
\begin{equation}
    \rbv{\reals^d} = \left\{f : \reals^d \to \reals \text{ is measurable} : \begin{aligned} & \rtv{f} < \infty, \\ & \esssup_{\bm{x} \in \reals^d} |f(\bm{x})|(1 + \|\bm{x}\|)^{-1} < \infty \end{aligned}\right\},
\end{equation}
where the second-order Radon total variation norm $\rtv{f}$ is a seminorm defined by
\begin{equation}
    \rtv{f} = c_d \|\partial_t^2\Lambda^{d-1}\mathcal{R}f\|_{\mathcal{M}(\mathbb{S}^{d-1}\times\mathbb{R})}. \label{eq: rnorm}
\end{equation}
Here, $\Lambda^{d-1} = (-\partial_t^2)^{\frac{d-1}{2}}$, $c_d^{-1} = 2(2\pi)^{d-1}$ is a dimension-dependent constant, and $\|\cdot\|_{\mathcal{M}(\mathbb{S}^{d-1}\times\mathbb{R})}$ denotes the total variation norm in the sense of measures supported on $\mathbb{S}^{d-1}\times\mathbb{R}$. Note that all operators must be understood 
in the distributional sense (see \cite{Parhi2020BanachSR,parhi2024distributional} for more details). The seminorm in \eqnref{eq: rnorm} exactly coincides with the representational cost of a function realized as a single-hidden-layer bounded-norm infinite-width network and coincides with the $\cR$-norm 
introduced by \citet{Ongie2020A}.

%% file: R-norm.tex
\section{$\mathcal{R}\mathrm{TV}^2$ of a Kernel Machine}\label{sec: rnorm}
In this section, we study the $\mathcal{R}\mathrm{TV}^2$ of kernel machines in the RKHS $\cH$. We show that one can write an explicit computable form for the case when the input dimension $d$ is odd. 
Consider the underlying matrix $\bM \succ 0$ for the Gaussian kernel $k_{\bM}$ has the following Cholesky decomposition 
\begin{equation}
\bM = \bL\bL^{\mathsf{T}}.
\end{equation}
Since $\bM$ is full rank and is in $\symmp$ this decomposition is unique. With this we state the following result on $\rtv{f}$ of a kernel machine $f \in \cH(\reals^d)$ with the proof in \iftoggle{longversion}{\appref{app: explicit}}{the supplemental materials}.


\begin{theorem}\label{thm: rtv}
    Assume that the input dimension $d$ is odd.
    For a kernel machine $f \in \cH(\reals^d)$ of the form
    \begin{align}
        f(\cdot) = \sum_{i =1}^k \alpha_i k_{\mathbf{M}}(\bm{x}_i,\cdot),
    \end{align}
    the $\mathcal{R}\mathrm{TV}^2$ of $f$ is given by
    \begin{align}
    \mathcal{R}\mathrm{TV}^2(f) = \frac{1}{\abs{\det \mathbf{L}}} \frac{1}{\sqrt{2\pi}} \int_{\mathbb{S}^{d-1}} \frac{1}{{\norm{\mathbf{L}^{-\mathsf{T}}\bm{\beta}}}} 
    \int_\mathbb{R}
    \abs{\sum_{i=1}^k \alpha_i \paren{\frac{\partial^{d+1}}{\partial t^{d+1}} \exp\paren{-\frac{(t-\bm{x}_i^\mathsf{T}\bm{\beta})^2}{2\norm{\mathbf{L}^{-\mathsf{T}}\bm{\beta}}^2}}}} \,\mathrm{d}t\,\mathrm{d}\bm{\beta},        
    \end{align}
where we have used the decomposition $\mathbf{M} = \mathbf{L}^\mathsf{T}\mathbf{L}$.
Furthermore, this can be extended to the case when $f$ has a representation with infinite kernel functions by taking limits.
\end{theorem}

\begin{proof}{\textbf{Outline}}
The proof proceeds in three main steps:
First, we leverage the factorization $\mathbf{M}=\mathbf{L}^\mathsf{T}\mathbf{L}$ to express the Gaussian kernel for a single center $\bm{x}_0$ as
\begin{align}
g(\bm{x})=\frac{1}{(2\pi)^{d/2}}\exp\Bigl(-\frac{|\mathbf{L}(\bm{x}-\bm{x}_0)|^2}{2}\Bigr).    
\end{align}
Next, we compute its Fourier transform using the change-of-variables formula to obtain
\begin{align}
\hat{g}(\bm{\omega})=\exp\Bigl(-\mathrm{i}\bm{x}_0^\mathsf{T}\bm{\omega}\Bigr)\frac{1}{|\det \mathbf{L}|}\exp\Bigl(-\frac{|\mathbf{L}^{-\mathsf{T}}\bm{\omega}|^2}{2}\Bigr).    
\end{align}
Finally, we apply the Fourier slice theorem~\citep{RammRadonBook} 
to connect the one-dimensional Fourier transform of $\mathcal{R}\{g\}(\bm{\beta},t)$ (with respect to $t$) with $d$-variate Fourier transform evaluated on one slice: $\hat{g}(\omega\bm{\beta})$. By inverting this transform, we derive the explicit expression for $\mathcal{R}\{g\}(\bm{\beta},t)$.
For odd dimensions $d$, the second-order Radon total variation of smooth functions is characterized by the $L_1$-norm of the $(d+1)$th $t$-derivative of $\mathcal{R}\{g\}(\bm{\beta},t)$ \citep[cf.,][]{Ongie2020A,Parhi2020BanachSR}. The result then readily extends to any finite kernel machine
\begin{align}
f(\cdot)=\sum_{i=1}^k \alpha_i k_\mathbf{M}(\bm{x}_i,\cdot)    
\end{align}
through the linearity of both the Fourier and Radon transforms.
\end{proof}

\subsection{$\mathcal{R}\mathrm{TV}^2$ as an Expression of Hermite Polynomials}
In \secref{sec: setup}, we discussed Hermite polynomials (probabilist's). In the following, we show how \thmref{thm: rtv} can be rewritten in terms of Hermite polynomials. In the next section, we study certain useful property of this expression to show the construction of a diverging $\mathcal{R}\mathrm{TV}^2$ sequence of kernel machines.



First, consider the the case of a $g \in \cH$ defined on one center $\bm{x}_0 \in \reals^d$.
Using \thmref{thm: rtv} we can write the $\mathcal{R}\mathrm{TV}^2$-norm of the kernel machine $g$ for one center $\bm{x}_0$ as
\begin{equation}
\mathcal{R}\mathrm{TV}^2(g) = \frac{1}{\abs{\det \mathbf{L}}} \frac{1}{\sqrt{2\pi}} \int_{\mathbb{S}^{d-1}} \frac{1}{{\norm{\mathbf{L}^{-\mathsf{T}}\bm{\beta}}}} 
    \int_\mathbb{R}
    \abs{ \paren{\frac{\partial^{d+1}}{\partial t^{d+1}} \exp\paren{-\frac{(t-\bm{x}_0^\mathsf{T}\bm{\beta})^2}{2\norm{\mathbf{L}^{-\mathsf{T}}\bm{\beta}}^2}}}} \,\mathrm{d}t\,\mathrm{d}\bm{\beta}. \label{eq: tv}
\end{equation}


First, consider the inner integral in $\mathcal{R}\mathrm{TV}^2(g)$ and denote it as
\begin{align}
I(\bm{\beta}) := \int_\mathbb{R}
    \abs{ \paren{\frac{\partial^{d+1}}{\partial t^{d+1}} \exp\paren{-\frac{(t-\bm{x}_0^\mathsf{T}\bm{\beta})^2}{2\sigma^2}}}} \mathrm{d}t,    
\end{align}
where we use $\sigma = \norm{\mathbf{L}^{-\mathsf{T}}\bm{\beta}}$.

Now, denote $\mu := \bm{x}_0\bm{\beta}$. Then, we note that the $(d+1)$-th derivative of $\exp\paren{-\frac{(t-\mu)^2}{2\sigma^2}}$ is related to the $(d+1)$-th Hermite polynomial $H_{d+1}$ as follows: 
\begin{align}
\frac{\partial^{d+1}}{\partial t^{d+1}} \exp\paren{-\frac{(t-\mu)^2}{2\sigma^2}} = (-1)^{d+1} \sigma^{-(d+1)} H_{d+1}\paren{\frac{t-\mu}{\sigma}} \exp\paren{-\frac{(t-\mu)^2}{2\sigma^2}}.   
\end{align}
\newline
Now, let $u = \frac{t - \mu}{\sigma}$. Thus, $\mathrm{d}u = \frac{1}{\sigma} \mathrm{d}t$. Substituting this transformation to $I(\bm{\beta})$ gives
\begin{align}
I(\bm{\beta}) = \int_\mathbb{R}
    \abs{ \paren{\frac{\partial^{d+1}}{\partial t^{d+1}} \exp\paren{-\frac{(t-\mu)^2}{2\sigma^2}}}} \mathrm{d}t &= \int_\mathbb{R} \abs{ (-1)^{d+1} \sigma^{-(d+1)} H_{d+1}(u) e^{-\frac{u^2}{2}}} \sigma \mathrm{d}u \label{eq: singleint1}\\
    &= \sigma^{-d} \int_\mathbb{R} \abs{H_{d+1}(u) e^{-\frac{u^2}{2}}} \mathrm{d}u. \label{eq: singleint2}
\end{align}
We can rewrite $I(\bm{\beta})$ as
\begin{align}
I(\bm{\beta}) = \sigma^{-d} \int_\mathbb{R} \abs{H_{d+1}(u) e^{-\frac{u^2}{2}}} \mathrm{d}u = \sigma^{-d} C_d,    
\end{align}
where $C_d :=  \int_\mathbb{R} \abs{H_{d+1}(u) e^{-\frac{u^2}{2}}} \mathrm{d}u$. In \secref{sec: diverge}, we bound this quantity to achieve certain decay of an infinite sum.

Replacing the computation in \eqnref{eq: singleint2} to \eqnref{eq: tv} gives
\begin{align}
    \mathcal{R}\mathrm{TV}^2(g) = \frac{C_d}{\abs{\det \mathbf{L}}} \frac{1}{\sqrt{2\pi}} \int_{\mathbb{S}^{d-1}} \frac{1}{{\norm{\mathbf{L}^{-\mathsf{T}}\bm{\beta}}^{d+1}}} \,\mathrm{d}\bm{\beta}. \label{eq: simprtv}
\end{align}
Thus, this shows that the expression of $\rtv{g}$ in \thmref{thm: rtv} can be simplified in terms of Hermite polynomials. 

In the following we extend this for $k > 1$ with the proof deferred to \iftoggle{longversion}{\appref{app: cov}}{the supplemental materials}.
\begin{lemma}\label{lem: cov}
 For a kernel machine $f \in \cH$ in the space of Gaussian RKHS. If $f$ has the following representation
    \begin{align}
        f(\cdot) = \sum_{i =1}^k \alpha_i k_{\mathbf{M}}(\bm{x}_i,\cdot)
    \end{align}
    for a center set $\curlybracket{\x_1, \x_2,\ldots, \x_k}$. Then, we can write
    \begin{gather}
\rtv{f} = \frac{1}{|\det \mathbf{L}| \sqrt{2\pi}} \int_{\mathbb{S}^{d-1}} \frac{I_{\sf{inner}}(\bm{\beta})}{\sigma^{d+1}} \, d\bm{\beta}, \label{eq:transformed_I}\\
I_{\sf{inner}}(\bm{\beta}) := \int_{\mathbb{R}} \left| \sum_{i=1}^k \alpha_i H_{d+1}\left( y + \Delta_i \right) e^{ -\frac{(y + \Delta_i)^2}{2} } \right| dy, \label{eq:transformedinner}
\end{gather}
where $\sigma = \|\mathbf{L}^{-\mathsf{T}}\bm{\beta}\|$ and 
\begin{align}
\Delta_i = \frac{\bm{x}_1^\mathsf{T}\bm{\beta} - \bm{x}_i^\mathsf{T}\bm{\beta}}{\|\mathbf{L}^{-\mathsf{T}}\bm{\beta}\|}
\quad \text{for } i = 2, 3, \ldots, k,    
\end{align}
and \( \Delta_1 = 0 \).
\end{lemma}


%% file: infinite_RKHS.tex
\section{A Sequence of Kernel Machines with Diverging $\mathcal{R}\mathrm{TV}^2$}\label{sec: diverge}
    
    
    
    
    
    
\begin{figure}
    \centering
    \begin{tikzpicture}[line cap=round,line join=round]

  \path[fill=green!20,draw=black,thick]
    (0,0)                
    -- (4,-1)            
    arc [ start angle=-90, end angle=90,
           x radius=0.4, y radius=1, 
           xshift=4cm, yshift=0cm ]
    -- cycle;            
  
  \draw[dashed] (4,0) ellipse [x radius=0.4, y radius=1];
  
  \node[left] at (0,0) {(0,0)};

    \node at (4.,-0.35) {$\eta$};
  \draw[->,blue,thick]
    (0,0)
    -- (5.5,0) 
    node[above,black] {$\bm{\beta}$};

    \draw[->,blue,dashed,thick]
    (0,0)
    -- (5,0.5) 
    node[above,black] {$\bm{\beta}'$};

    \fill[red,thick] (1.5,0.6) circle (1.6pt) node[left]{$\bm{x}_1$};
    \fill[red,thick] (2.5,1.0) circle (1.6pt) node[above]{$\bm{x}_2$};
    
    \fill[purple] (3.5,2.4) circle (1.6pt) node[right]{$\bm{x}_3$};
    \fill[purple] (5.5,2.1) circle (1.6pt) node[right]{$\bm{x}_4$};

    \draw[dashed,gray,thick] (0,0) -- (1.5,0.6) node[left]{};
     \draw[dashed,gray,thick] (0,0) -- (3.5,2.4) node[left]{};

    \node at (1.0,-0.43) {$\delta_1$};
    \node at (1.78,-0.658) {$\delta_{12}$};
    \node at (2.63,-0.808) {$\delta_{23}$};
    \draw[thick,black,thick] (1.5,0.6) -- (1.24,-0.33) node[left]{};
    \draw[thick,black,thick] (2.5,1.0) -- (2.15,-0.508) node[left]{};
    \draw[thick,black,thick] (3.5,2.4) -- (2.83,-0.708) node[left]{};

\end{tikzpicture}

    \caption{Illustration of an $(\bm{\beta}, \delta,\eta)$-separated set and a sequence $(\bm{x}_1,\bm{x}_2,\bm{x}_3,\bm{x}_4)$ that satisfy the requirements of the definition. The distances $\delta_1, \delta_{12}, \delta_{23}$ are at least $\delta$ apart.}
    \label{fig: separated}
\end{figure}
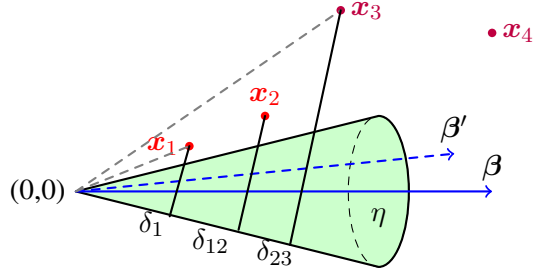

In this section, we construct a sequence of kernel machines $\curlybracket{f_n \in \cH(\reals^d)}$ such that their $\mathcal{R}\mathrm{TV}^2$ diverges. 
%
First, we state some useful assumptions on the probabilist's Hermite polynomial which are easy to verify to hold in general (but surely in odd dimension $d$).

\begin{assumption}[$\delta$-peak]\label{ass: peak}
Fix a dimension $d$. For a given Hermite polynomial $H_{d+1}$, we call an interval $[-\delta,\delta]$ a region of $\delta$-peak if:
\begin{enumerate}
    \item $\frac{\partial H_{d+1}(y)e^{-\frac{y^2}{2}}}{\partial y} < 0$ for all $y > \delta$
    \item $\frac{\partial H_{d+1}(y)e^{-\frac{y^2}{2}}}{\partial y} > 0$ for all $y < -\delta$
\end{enumerate}
\end{assumption}
Due to exponential decay of the product $H_{d+1}(y)e^{-\frac{y^2}{2}}$, for any odd dimension $d$ note that 
\begin{align}
  \frac{\partial H_{d+1}(y)e^{-\frac{y^2}{2}}}{\partial y} = (H'_{d+1}(y) - yH_{d+1}(y))e^{-y^2/2},  
\end{align}
where $-yH_{d+1}(y)$ is a polynomial with odd dimension with negative highest term and this implies there exists a $\delta$-peak.
Now, we state a trivial observation on the absolute integral of $H_{d+1}\left( y \right) \, e^{ -\frac{y^2}{2} }$. 
\begin{assumption}[$\epsilon$-safe]\label{ass: safe}
    We say a constant $\epsilon > 0$ is $\epsilon$-safe if 
    \begin{align}
      \int_{[-\epsilon,\epsilon]} \left| H_{d+1}\left( y \right) \, e^{ -\frac{y^2}{2} }\right| \, dy > 0.  
    \end{align}
\end{assumption}
Since $H_{d+1}$ is non-zero polynomial this holds trivially for any $\epsilon > 0$. Furthermore, the integral is increasing with the size of an $\epsilon$-interval. With this, we state a useful result on the convergence of a series of evaluations of $H_{d+1}\left( y \right) \, e^{ -\frac{y^2}{2}}$ on distinct points $y \in \reals$. The proof appears in \iftoggle{longversion}{\appref{app: useful}}{the supplemental materials}.



\begin{lemma}\label{lem: inter}
Let \( d\ge 0 \) be fixed and let \( H_{d+1}(y) \) denote the Hermite polynomial of degree \( d+1 \). Then for any constant $\rho > 0$, there exists a constant \( \delta_0 > 0 \) (depending only on \( d \)) such that for every \(\delta \ge \delta_0\) we have
\begin{align}
\sum_{j=2}^\infty \Bigl| H_{d+1}\bigl(j\delta\bigr) \Bigr|\, e^{-\frac{(j\delta)^2}{2}} < \frac{\rho}{4}.    
\end{align}
\end{lemma}

\subsection{Construction of a Diverging Sequence}

In \secref{sec: setup}, we defined the notions of $(\bm{\beta}, \delta, \eta)$-separated sets of size $n \in \mathbb{N}$. Let $(\bm{x}_1,\bm{x}_2,\ldots,\bm{x}_n)$ be a sequence in this set. Intuitively, any two centers in the sequence are at least $\delta$ apart when projected onto any direction $\bm{\beta}'$ such that $\bm{\beta}^{\mathsf{T}}\bm{\beta}' \ge \norm{\bm{\beta}}\norm{\bm{\beta}'}\eta$ (see \figref{fig: separated} for an illustration). Now, note that in \lemref{lem: cov}, we provided an alternate representation of the $\rtv{f}$ of a function as shown in \thmref{thm: rtv}, specifically the inner integral for each $\bm{\beta} \in \mathbb{S}^{d-1}$ has the form:
\begin{align}
  I_{\sf{inner}}(\bm{\beta}) := \int_{\mathbb{R}} \left| \sum_{i=1}^k \alpha_i H_{d+1}\left( y + \Delta_i \right) e^{ -\frac{(y + \Delta_i)^2}{2} } \right| dy,  
\end{align}
where each $\Delta_i = \bm{x}_1^\mathsf{T}\bm{\beta} - \bm{x}_i^\mathsf{T}\bm{\beta}$ (ignoring the normalization). If the projections $\bm{x}_i^\mathsf{T}\bm{\beta}$ are far apart on the real line $\reals$, noting the absolute decay in the values of $H_{d+1}\left( y \right) \, e^{ -\frac{y^2}{2} }$ outside the region of $\delta$-peak as asserted by \assref{ass: peak}, we can quantify and control contributions of terms corresponding to $j \neq i$ in the inner integral. 

Now, the property holds over a non-trivial cone $\cK(\bm{\beta}) := \curlybracket{\bm{\beta}' \in \mathbb{S}^{d-1}\,|\, \bm{\beta}^{\mathsf{T}}\bm{\beta}' \ge \eta }$ with non-zero volume. Now, note in \eqnref{eq:transformed_I}, which involves the following integral
\begin{align}
  \int_{\mathbb{S}^{d-1}} \frac{I_{\sf{inner}}(\bm{\beta})}{\sigma^{d+1}} \, d\bm{\beta}  
\end{align}
is non-trivially positive. Thus, we show along any direction $\bm{\beta}$ in the cone, $I_{\sf{inner}}$ diverges as $k$ grows if the kernel machine $f$ is defined for a sequence of centers from the $(\bm{\beta}, \delta, \eta)$-separated set.

 Now, we state the main theorem of the work. For ease of analysis we assume that the largest eigenvalue of $\mathbf{L}^{-1}$ is upper bounded by 1 which can be easily replaced with appropriate rescaling and choice of the parameters in the statement. 
 
\begin{theorem}[Diverging $\mathcal{R}\mathrm{TV}^2$]\label{thm: unbounded}
    Consider the Gaussian RKHS $\cH(\reals^d)$ as defined in \eqnref{eq: infinitegauss}. Assume $\epsilon \in (0,1/2]$ be a safe constant (see \assref{ass: safe}). Define 
    \begin{align}
        \rho :=  \int_{[-\epsilon,\epsilon]} \left| H_{d+1}\left( y \right) \, e^{ -\frac{y^2}{2} }\right| \, dy.
    \end{align}
    Fix a unit vector $\bm{\beta} \in \reals^d$, scalars $\eta \ge \frac{\sqrt{3}}{2}$, and $\delta = 3\max\{\epsilon, \delta_0(\rho), \delta'\}$ where $\delta_0(\rho)$ is chosen as per \lemref{lem: inter}, and $\delta'(d)$ as per \assref{ass: peak}. Let $\cX_\infty = \curlybracket{\bm{x}_1,\bm{x}_2,\ldots} \subset \reals^d$ be an infinite sequence such that any subsequence $\Gamma_n = \curlybracket{\bm{x}_1,\bm{x}_2,\ldots, \bm{x}_n}$ is in the $(\bm{\beta}, \delta, \eta)$-separated set of size $n$. Define a function $f \in \cH$ on $\cX_\infty$ that has a representation with an $f$-unbounded combination $\alpha_f$. Then,
    \begin{align}
        \rtv{\{f_n\}} \to \infty,
    \end{align}
    as $n \to \infty$.
\end{theorem}

\begin{proof}
    First, we rewrite \eqnref{eq:transformed_I} for the $\mathcal{R}\mathrm{TV}^2$ of the function $f_k$ as follows
    \begin{align}
        \rtv{f_k} &= \frac{1}{|\det \mathbf{L}| \sqrt{2\pi}} \int_{\mathbb{S}^{d-1}} \frac{1}{\sigma^{d+1}} \int_{\mathbb{R}} \left| \sum_{i=1}^k \alpha_i H_{d+1}\left( y + \Delta_i \right) e^{ -\frac{(y + \Delta_i)^2}{2} } \right| dy \, d\bm{\beta},
    \end{align}
    with the inner integral
    \begin{align}
      I_{\sf{inner}}(\bm{\beta}) = \int_{\reals} \left| \sum_{i=1}^k \alpha_i \, H_{d+1}\left( y + \Delta_i \right) \, e^{ -\frac{(y + \Delta_i)^2}{2} } \right| \, dy,  
    \end{align}
where
\begin{align}
\Delta_1 = 0,\,\,\Delta_i = \frac{a_1 - a_i}{\sigma} = \frac{\bm{x}_1^\mathsf{T}\bm{\beta} - \bm{x}_i^\mathsf{T}\bm{\beta}}{\|\mathbf{L}^{-\mathsf{T}}\bm{\beta}\|}
\quad \text{for } i = 2, 3, \ldots, k.    
\end{align}
First, define the cone $\cK$ wrt $\bm{\beta}$ and $\eta$ as stated in the theorem statement, i.e.,
\begin{align}
    \cK := \curlybracket{\bm{\beta}' \in \mathbb{S}^{d-1}\:\middle|\: \bm{\beta}'^{\mathsf{T}} \bm{\beta} \ge \eta}.
\end{align}
Note that the volume $\vol{\cK} > 0$ implying that 
\begin{align}
    \int_{\mathbb{S}^{d-1}} \frac{1}{\sigma^{d+1}} \, d\bm{\bm{\beta}} \ge  \int_{\cK} \frac{1}{\sigma^{d+1}} \, d\bm{\beta} = (\star). \label{eq: cone}
\end{align}
    Note that $\mathbf{M} = \mathbf{L}^\mathsf{T}\mathbf{L}$. We assume that the $\mathbf{M}$ is symmetric and PSD, implying that singular values of $\mathbf{L}$ are \tt{exactly} the square root of the eigenvalues of $\mathbf{M}$, i.e.,
    \begin{align}
       \sigma_i(\mathbf{L}) = \sqrt{|\lambda_i(\mathbf{M})|}. 
    \end{align}
    But, since $\mathbf{L}$ is invertible implying the singular values if $\mathbf{L}^{-1}$ are inverses to singular values of $\mathbf{L}$, i.e., $\sigma_i(\mathbf{L}^{-1}) = \frac{1}{\sigma_i(\mathbf{L})}$.
    Thus, we can rewrite \eqnref{eq: cone} as
    \begin{align}
     (\star) & =  \int_{\cK} \frac{1}{\sigma_{\max}(\mathbf{L}^{-1})^{d+1}} \, d\bm{\beta} \nonumber\\
     &\ge \int_{\cK} \sigma_{\min}(\mathbf{L})^{d+1} \, d\bm{\beta} \nonumber \\ 
     &= \int_{\cK} \lambda_{\min}(\mathbf{M})^{\frac{(d+1)}{2}} \, d\bm{\beta} =  \lambda_{\min}(\mathbf{M})^{\frac{(d+1)}{2}} \vol{\cK} > 0. \label{eq: vol}
    \end{align}
    Now, we will show for any $\bm{\beta}' \in \cK$, there is a non-trivial lower bound on $I_{\sf{inner}}(\bm{\beta}')$. Note that by definition, $\Gamma_n$ is in the $(\bm{\beta}',\delta)$-separated set. 

    Hence, for all $i,j = 2, 3, \ldots$
    \begin{align}
        |\Delta_i - \Delta_j| \ge \delta.
    \end{align}
    Define the neighborhoods $\curlybracket{N_i}$ for the safe constant $\epsilon$ as follows
    \begin{align}
      N_i := \bracket{-\Delta_i - \epsilon, -\Delta_i + \epsilon}.  
    \end{align}
    Now, consider the integral on the neighborhood $N_i$:
    \begin{align}
        &\int_{N_i} \left| \sum_{i=1}^k \alpha_i \, H_{d+1}\left( y + \Delta_i \right) \, e^{ -\frac{(y + \Delta_i)^2}{2} } \right| \, dy \nonumber \\
        &\ge \int_{N_i} \left|\alpha_i \, H_{d+1}\left( y + \Delta_i \right) \, e^{ -\frac{(y + \Delta_i)^2}{2} } \right| \, dy - \underbrace{\int_{N_i} \left| \sum_{j=1, j \neq i}^k \alpha_j \, H_{d+1}\left( y + \Delta_j \right) \, e^{ -\frac{(y + \Delta_j)^2}{2} } \right| \, dy}_{=: \large{\theta_i}}.
    \end{align}
    The second line follows from the triangle inequality. Now, change of variable simplifies the first equation as 
    \begin{align}
        \int_{N_i} \left|\alpha_i \, H_{d+1}\left( y + \Delta_i \right) \, e^{ -\frac{(y + \Delta_i)^2}{2} } \right| \, dy 
        \ge |\alpha_i| \int_{[-\epsilon,\epsilon]} \left|H_{d+1}\left( y \right) \, e^{ -\frac{y^2}{2} } \right| \, dy  
        \ge |\alpha_i| \rho. 
    \end{align}
    In the last equation, we used the definition of $\rho$.

    Now, summing over each $i = 1, 2, \ldots k$, we get
    \begin{align}
        I_{\sf{inner}} \ge \sum_{i = 1}^k \int_{N_i} \left| \sum_{i=1}^k \alpha_i \, H_{d+1}\left( y + \Delta_i \right) \, e^{ -\frac{(y + \Delta_i)^2}{2} } \right| \, dy 
         \ge \rho \sum_{i=1}^k |\alpha_i| - \sum_{i=1}^k \theta_i. \label{eq: diff}
    \end{align}
    Now, we will show how to bound the sum $\sum_{i=1}^k \theta_i$.
    \paragraph{Bounding $\theta_k$:} First note that, we can bound each $\theta_i$ as follows
    \begin{align}
        \theta_i &= \int_{N_i} \left| \sum_{j=1, j \neq i}^k \alpha_j \, H_{d+1}\left( y + \Delta_j \right) \, e^{ -\frac{(y + \Delta_j)^2}{2} } \right| \, dy\nonumber\\ &\le  \sum_{j=1, j \neq i}^k \int_{N_i} \left|  \alpha_j \, H_{d+1}\left( y + \Delta_j \right) \, e^{ -\frac{(y + \Delta_j)^2}{2} } \right| \, dy  \label{eq:theta1}\\ 
        &\le  \sum_{j=1, j \neq i}^k \int_{\bracket{-\epsilon,\epsilon}} \left|  \alpha_j \, H_{d+1}\left(-\Delta_i + y + \Delta_j \right) \, e^{ -\frac{(-\Delta_i + y + \Delta_j)^2}{2} } \right| \, dy \label{eq:theta2}\\ 
        & \le \sum_{j=1, j \neq i}^k |\alpha_j| \int_{\bracket{-\epsilon,\epsilon}}  \left| H_{d+1}\left( |i-j|\delta \right) \, e^{ -\frac{(|i-j| \delta)^2}{2} } \right| \, dy \label{eq:theta3}\\ 
        & \le 2 \epsilon \sum_{j=1, j \neq i}^k |\alpha_j| \left|H_{d+1}\left( |i-j|\delta \right)\right| \, e^{ -\frac{(|i-j| \delta)^2}{2} }\nonumber \\
        & \le \sum_{j=1, j \neq i}^k |\alpha_j| \left|H_{d+1}\left( |i-j|\delta \right)\right| \, e^{ -\frac{(|i-j| \delta)^2}{2} }. \label{eq:theta4} 
    \end{align}
    \eqnref{eq:theta1} is a straight-forward application of triangle inequality. In \eqnref{eq:theta2}, we simplify \eqnref{eq:theta1} via change of variable. In \eqnref{eq:theta3}, we use the assumption of $\delta$-peak. To simplify $-\Delta_i + y + \Delta_j$ for a choice of $y \in \bracket{-\epsilon,\epsilon}$, we assume that indices of the projections $\Delta_i$ for $i = 1,2, \ldots$ are arranged in ascending order in their values on the real line. Since each consecutive projections are at least $\delta$ apart, we can bound $|-\Delta_i + y + \Delta_j| > (|i-j| - 1/3)\delta$. Since Hermite polynomials in even dimension, i.e. $d+1$, are even 
    \begin{align}
      \left|H_{d+1}\left(-\Delta_i + y + \Delta_j \right) \, e^{ -\frac{(-\Delta_i + y + \Delta_j)^2}{2}}\right| \le \left| H_{d+1}\left((|i - j| - 2/3)\delta\right) e^{ -\frac{((|i - j| - 2/3)\delta)^2}{2}} \right|. 
    \end{align}
    For simplification, we have omitted the $-(2/3)\delta$ additive term in the equation above. Finally, \eqnref{eq:theta4} follows as $\epsilon \le 1/2$.

    Summing over each $i = 1,\ldots,k$
    \begin{align}
        \sum_{i=1}^k \theta_k &= \sum_{i=1}^k \int_{N_i} \left| \sum_{j=1, j \neq i}^k \alpha_j \, H_{d+1}\left( y + \Delta_j \right) \, e^{ -\frac{(y + \Delta_j)^2}{2} } \right| \, dy \nonumber \\
        &\le \sum_{i=1}^k \sum_{j=1, j \neq i}^k |\alpha_j| \left|H_{d+1}\left( |i-j|\delta \right)\right| \, e^{ -\frac{(|i-j| \delta)^2}{2} } \nonumber \\
        &\le 2\paren{\sum_{j=1}^{k}  \left|H_{d+1}\left( j\delta \right) \right|\, e^{ -\frac{(j \delta)^2}{2} }} \sum_{i=1}^k |\alpha_i|.
    \end{align}
    Using \lemref{lem: inter}, we can rewrite \eqnref{eq: diff} as
    \begin{align}
        I_{\sf{inner}} \ge \rho \sum_{i=1}^k |\alpha_i| - 2\paren{\sum_{j=2}^k  \left|H_{d+1}\left( j\delta \right)\right| \, e^{ -\frac{(j \delta)^2}{2} }} \sum_{i=1}^k |\alpha_i| \ge \frac{\rho}{2}\sum_{i=1}^k |\alpha_i|.
    \end{align}

    \noindent Now, note that using \eqnref{eq: cone} and \eqnref{eq: vol}
    \begin{align}
        \rtv{f_k} &\ge \frac{1}{|\det \mathbf{L}| \sqrt{2\pi}} \int_{\cK} \frac{1}{\sigma^{d+1}} \int_{\mathbb{R}} \left| \sum_{i=1}^k \alpha_i H_{d+1}\left( y + \Delta_i \right) e^{ -\frac{(y + \Delta_i)^2}{2} } \right| dy \, d\bm{\beta} \nonumber \\
        & \ge \frac{1}{|\det \mathbf{L}| \sqrt{2\pi}} \int_{\cK} \frac{1}{\sigma^{d+1}}  \paren{\frac{\rho}{2}\sum_{i=1}^k |\alpha_i| } \, d\bm{\beta} \nonumber \\
        & \ge \paren{\frac{1}{|\det \mathbf{L}| \sqrt{2\pi}} \lambda_{\min}(\mathbf{M})^{\frac{(d+1)} {2}} \vol{\cK}} \cdot\paren{\frac{\rho}{2}\sum_{i=1}^k |\alpha_i| }.
    \end{align}
    Now, in the limiting case
    \begin{align}
        \lim_{k \to \infty} \rtv{f_k} &\ge \paren{\frac{1}{|\det \mathbf{L}| \sqrt{2\pi}} \lambda_{\min}(\mathbf{M})^{\frac{(d+1)} {2}} \vol{\cK}} \lim_{k \to \infty} \paren{\frac{\rho}{2}\sum_{i=1}^k |\alpha_i| } \nonumber \\
        & =  \paren{\frac{1}{|\det \mathbf{L}| \sqrt{2\pi}} \lambda_{\min}(\mathbf{M})^{\frac{(d+1)} {2}} \vol{\cK}} \cdot \frac{\rho}{2} \norm{\alpha_f} \to \infty.
    \end{align}
    Hence the claim of the theorem has been proven.
\end{proof}


%% file: app_explicit.tex






\section{$\mathcal{R}\mathrm{TV}^2$ for centers size $k > 1$}\label{app: explicit}
In this Appendix, we provide the proof of \thmref{thm: rtv}.
First, we provide the proof for one center and then extend it to multi-center settings.

\begin{proof}
Let
\begin{align}
    g(\bm{x}) = \frac{1}{(2\pi)^{d/2}} \exp\paren{-\frac{\norm{\bm{x} - \bm{x}_0}_\mathbf{M}^2}{2}}.
\end{align}
Also, define the $\bm{0}$ mean identity covariance Gaussian
\begin{align}
    g_0(\bm{x}) = \frac{1}{(2\pi)^{d/2}} \exp\paren{-\frac{\norm{\bm{x}}^2}{2}}.
\end{align}
If we can write $\mathbf{M} = \mathbf{L}^\mathsf{T}\mathbf{L}$, then we have that
\begin{align}
\norm{\bm{x} - \bm{x}_0}_\mathbf{M}^2 = \norm{\mathbf{L}(\bm{x} - \bm{x}_0)}^2,
\end{align}
in which case
\begin{align}
g(\bm{x}) = \frac{1}{(2\pi)^{d/2}} \exp\paren{-\frac{\norm{\mathbf{L}\bm{x} - \mathbf{L}\bm{x}_0}^2}{2}}.
\end{align}
We have the Fourier transform
\begin{align}
    \hat{g}_0(\bm{\omega}) = \exp\paren{-\frac{\norm{\bm{\omega}}^2}{2}}.
\end{align}
We have the equality $g(\bm{x}) = g_0(\mathbf{L}\bm{x} - \mathbf{L}\bm{x}_0)$. Using the change of variables formula for the Fourier transform, we have
\begin{align}
    \hat{g}(\bm{\omega})
    &= \exp\paren{-\mathrm{i}(\mathbf{L}\bm{x}_0)^\mathsf{T}\mathbf{L}^{-\mathsf{T}}\bm{\omega}} \frac{1}{\abs{\det \mathbf{L}}} \hat{g}_0(\mathbf{L}^{-\mathsf{T}} \bm{\omega}) \\
    &= \exp\paren{-\mathrm{i}\bm{x}_0^\mathsf{T}\mathbf{L}^\mathsf{T}\mathbf{L}^{-\mathsf{T}}\bm{\omega}} \frac{1}{\abs{\det \mathbf{L}}} \exp\paren{-\frac{\norm{\mathbf{L}^{-\mathsf{T}} \bm{\omega}}^2}{2}} \\
    &= \exp\paren{-\mathrm{i}\bm{x}_0^\mathsf{T}\bm{\omega}} \frac{1}{\abs{\det \mathbf{L}}} \exp\paren{-\frac{\norm{\mathbf{L}^{-\mathsf{T}} \bm{\omega}}^2}{2}}. \label{eq: firstfour}
\end{align}
The Fourier slice theorem~\citep{RammRadonBook} says that
\begin{align}
    \mathcal{F}_1\{\mathcal{R}\{f\}(\bm{\beta}, \cdot)\}(\omega) = \hat{f}(\omega\bm{\beta}).
\end{align}
If we evaluate $ \hat{g}$ at $\bm{\omega} = \omega\bm{\beta}$ in the \eqnref{eq: firstfour}, we find
\begin{align}
\hat{g}(\omega\bm{\beta})
&= \exp\paren{-\mathrm{i}\bm{x}_0^\mathsf{T}(\omega\bm{\beta})} \frac{1}{\abs{\det \mathbf{L}}} \exp\paren{-\frac{\norm{\mathbf{L}^{-\mathsf{T}} (\omega\bm{\beta})}^2}{2}} \\
&= \exp\paren{-\mathrm{i}(\bm{x}_0^\mathsf{T}\bm{\beta})\omega} \frac{1}{\abs{\det \mathbf{L}}} \exp\paren{-\frac{\abs{\omega}^2\norm{\mathbf{L}^{-\mathsf{T}} \bm{\beta}}^2}{2}}.
\end{align}
The 1D inverse Fourier transform of this is the Radon transform of $g$, i.e.,
\begin{align}
    \mathcal{R}\{g\}(\bm{\beta}, t) = \frac{1}{\abs{\det \mathbf{L}}} \frac{1}{\sqrt{2\pi}} \frac{1}{\sqrt{\norm{\mathbf{L}^{-\mathsf{T}}\bm{\beta}}^2}} \exp\paren{-\frac{(t- \bm{x}_0^\mathsf{T}\bm{\beta})^2}{2\norm{\mathbf{L}^{-\mathsf{T}}\bm{\beta}}^2}}.
\end{align}
If $d$ is odd, then the second-order Radon domain total variation is the $L_1$-norm of $(d+1)$ derivatives in $t$ of this quantity (see Equation (28) in \cite{Parhi2020BanachSR}). That is
\begin{align}
    \mathcal{R}\mathrm{TV}^2(g)
    &= \frac{1}{\abs{\det \mathbf{L}}} \frac{1}{\sqrt{2\pi}} \int_{\mathbb{S}^{d-1}} \int_\mathbb{R}\abs{\frac{1}{\sqrt{\norm{\mathbf{L}^{-\mathsf{T}}\bm{\beta}}^2}} \paren{\frac{\partial^{d+1}}{\partial t^{d+1}} \exp\paren{-\frac{(t-\bm{x}_0^\mathsf{T}\bm{\beta})^2}{2\norm{\mathbf{L}^{-\mathsf{T}}\bm{\beta}}^2}}}} \,\mathrm{d}t\,\mathrm{d}\bm{\beta} \\ 
    &= \frac{1}{\abs{\det \mathbf{L}}} \frac{1}{\sqrt{2\pi}} \int_{\mathbb{S}^{d-1}} \frac{1}{{\norm{\mathbf{L}^{-\mathsf{T}}\bm{\beta}}}} 
    \int_\mathbb{R}
    \abs{ \paren{\frac{\partial^{d+1}}{\partial t^{d+1}} \exp\paren{-\frac{(t-\bm{x}_0^\mathsf{T}\bm{\beta})^2}{2\norm{\mathbf{L}^{-\mathsf{T}}\bm{\beta}}^2}}}} \,\mathrm{d}t\,\mathrm{d}\bm{\beta}. 
\end{align}
This gives the stated expression on $\rtv{f}$ for one center.
\end{proof}

\subsection{Multi-Center Computation}

For a kernel machine with $k > 1$ centers, we can rewrite $g$ as
\begin{align}
g(\bm{x}) = \sum_{i =1}^k \frac{1}{(2\pi)^{d/2}}  \alpha_i \cdot \exp\paren{-\frac{\norm{\mathbf{L}\bm{x} - \mathbf{L}\bm{x}_i}^2}{2}}.
\end{align}
Denote by $g_i(\bm{x}) := \frac{1}{(2\pi)^{d/2}} \exp\paren{-\frac{\norm{\mathbf{L}\bm{x} - \mathbf{L}\bm{x}_i}^2}{2}}$ for each center $\bm{x}_i \in \cD$.

Now, the Fourier transform of $g$ can be written for the extended case, noting the linearity of the transform,
\begin{align}
 \hat{g}(\bm{\omega}) = \sum_{i=1}^k \hat{g}_i(\bm{\omega}).
\end{align}
This implies that 
\begin{align}
 \hat{g}(\bm{\omega}) = \sum_{i =1}^k \exp\paren{-\mathrm{i}\bm{x}_i^\mathsf{T}\bm{\omega}} \frac{1}{\abs{\det \mathbf{L}}} \exp\paren{-\frac{\norm{\mathbf{L}^{-\mathsf{T}} \bm{\omega}}^2}{2}}.
\end{align}
Now, computing the inverse Fourier transform of $\hat{g}$ wrt $\bm{\omega}$ gives 
\begin{align}
\mathcal{R}\{g\}(\bm{\beta}, t) = \frac{1}{\abs{\det \mathbf{L}}} \sum_{i=1}^k \alpha_i \frac{1}{\sqrt{2\pi}} \frac{1}{\sqrt{\norm{\mathbf{L}^{-\mathsf{T}}\bm{\beta}}^2}} \exp\paren{-\frac{(t- \bm{x}_i^\mathsf{T}\bm{\beta})^2}{2\norm{\mathbf{L}^{-\mathsf{T}}\bm{\beta}}^2}}.
\end{align}
As before the $\mathcal{R}\mathrm{TV}^2$ of $g$, i.e. the second-order Radon domain total variation for odd values of $d$ is the $L_1$-norm of $(d+1)$ derivatives in $t$ of this quantity. Thus,
\begin{align}
\mathcal{R}\mathrm{TV}^2(g) = \frac{1}{\abs{\det \mathbf{L}}} \frac{1}{\sqrt{2\pi}} \int_{\mathbb{S}^{d-1}} \frac{1}{{\norm{\mathbf{L}^{-\mathsf{T}}\bm{\beta}}}} 
    \int_\mathbb{R}
    \abs{\sum_{i=1}^k \alpha_i \paren{\frac{\partial^{d+1}}{\partial t^{d+1}} \exp\paren{-\frac{(t-\bm{x}_i^\mathsf{T}\bm{\beta})^2}{2\norm{\mathbf{L}^{-\mathsf{T}}\bm{\beta}}^2}}}} \,\mathrm{d}t\,\mathrm{d}\bm{\beta}.
\end{align}

This expression can be extended to the case of kernel machines with infinite centers by taking limits. In particular, it is guaranteed to be finite for the case when the $\ell_1$ norm of the coefficients $\alpha$ is finite since
\begin{align}
    \abs{\sum_{i} \alpha_i \paren{\frac{\partial^{d+1}}{\partial t^{d+1}} \exp\paren{-\frac{(t-\bm{x}_i^\mathsf{T}\bm{\beta})^2}{2\norm{\mathbf{L}^{-\mathsf{T}}\bm{\beta}}^2}}}} \le C\cdot \sum_{i}|\alpha_i|,
\end{align}
where it is straightforward to show that $\frac{\partial^{d+1}}{\partial t^{d+1}} \exp\paren{-\frac{(t-\bm{x}_i^\mathsf{T}\bm{\beta})^2}{2\norm{\mathbf{L}^{-\mathsf{T}}\bm{\beta}}^2}}$ is bounded by a universal constant $C > 0$ for all choices of $\bm{x}_i$.


%% file: app_change_of_variable.tex
\section{Change of Variables For Multiple Centers}\label{app: cov}
In this Appendix, we provide the proof of \lemref{lem: cov}.

\begin{proof}
    Previously, we computed the $\mathcal{R}\mathrm{TV}^2$ of a general kernel machine as
    \begin{align}
\mathcal{R}\mathrm{TV}^2(g) = \frac{1}{\abs{\det \mathbf{L}}} \frac{1}{\sqrt{2\pi}} \int_{\mathbb{S}^{d-1}} \frac{1}{{\norm{\mathbf{L}^{-\mathsf{T}}\bm{\beta}}}} 
    \int_\mathbb{R}
    \abs{\sum_{i=1}^k \alpha_i \paren{\frac{\partial^{d+1}}{\partial t^{d+1}} \exp\paren{-\frac{(t-\bm{x}_i^\mathsf{T}\bm{\beta})^2}{2\norm{\mathbf{L}^{-\mathsf{T}}\bm{\beta}}^2}}}} \,\mathrm{d}t\,\mathrm{d}\bm{\beta}.
\end{align}
Now, we can rewrite the ($d+1$)-th derivative of the involved exponential as follows
\begin{equation*}
\frac{\partial^{d+1}}{\partial t^{d+1}} \exp\left( -\frac{(t - a_i)^2}{2 \sigma^2} \right) = (-1)^{d+1} \sigma^{-(d+1)} H_{d+1}\left( \frac{t - a_i}{\sigma} \right) \exp\left( -\frac{(t - a_i)^2}{2 \sigma^2} \right),
\end{equation*}
where we define
\begin{align}
a_i = \bm{x}_i^\mathsf{T}\bm{\beta} \quad\text{and}\quad \sigma = \|\mathbf{L}^{-\mathsf{T}}\bm{\beta}\|.
\end{align}
Substituting the expression for Hermite polynomial (see \secref{sec: setup}) into the integral $I$ gives
\begin{align}
I &= \frac{1}{|\det \mathbf{L}| \sqrt{2\pi}} \int_{\mathbb{S}^{d-1}} \frac{1}{\sigma^{d+2}} \int_{\mathbb{R}} \left| \sum_{i=1}^k \alpha_i H_{d+1}\left( \frac{t - a_i}{\sigma} \right) e^{ -\frac{(t - a_i)^2}{2 \sigma^2} } \right| dt \, d\bm{\beta}. \nonumber
\end{align}
To simplify the inner integral, we can perform the following change of variable centered at \( a_1 = \bm{x}_1^\mathsf{T}\bm{\beta} \)
\begin{align}
y = \frac{t - a_1}{\sigma} \quad \Rightarrow \quad t = \sigma y + a_1 \quad \Rightarrow \quad dt = \sigma \, dy.
\end{align}
We can express all \( y_i \) in terms of \( y \) as follows
\begin{align}
y_i = \frac{t - a_i}{\sigma} = \frac{\sigma y + a_1 - a_i}{\sigma} = y + \Delta_i,
\end{align}
where we define
\begin{align}
\Delta_i = \frac{a_1 - a_i}{\sigma} = \frac{\bm{x}_1^\mathsf{T}\bm{\beta} - \bm{x}_i^\mathsf{T}\bm{\beta}}{\|\mathbf{L}^{-\mathsf{T}}\bm{\beta}\|}
\quad \text{for } i = 2, 3, \ldots, k,
\end{align}
and \( \Delta_1 = 0 \).

With this change of variable into the inner integral we have the stated final form 
\begin{align}
I &= \frac{1}{|\det \mathbf{L}| \sqrt{2\pi}} \int_{\mathbb{S}^{d-1}} \frac{1}{\sigma^{d+1}} \underbrace{\int_{\mathbb{R}} \left| \sum_{i=1}^k \alpha_i H_{d+1}\left( y + \Delta_i \right) e^{ -\frac{(y + \Delta_i)^2}{2} } \right| dy }_{\text{define $I_{\sf{inner}}$}}\, d\bm{\beta}
\end{align}
\begin{align}I_{\sf{inner}} = \int_{\reals} \left| \sum_{i=1}^k \alpha_i \, H_{d+1}\left( y + \Delta_i \right) \, e^{ -\frac{(y + \Delta_i)^2}{2} } \right| \, dy.
\end{align}
\end{proof}

%% file: app_usefulproperty.tex
\section{A Useful Property of Hermite Polynomials}\label{app: useful}
In this Appendix, we provide the proof of \lemref{lem: inter}.

\begin{proof}
Since \( H_{d+1}(y) \) is a polynomial of degree \( d+1 \), there exists a constant \( C > 0 \) (depending only on \( d \)) such that
\begin{align}
\Bigl| H_{d+1}(y) \Bigr| \le C \,(1+|y|)^{d+1} \quad \text{for all } y \in \mathbb{R}.
\end{align}
Hence, for any \(\delta>0\) and any integer \( j\ge 2 \) we have
\begin{align}
\Bigl| H_{d+1}\bigl(j\delta\bigr) \Bigr|\, e^{-\frac{(j\delta)^2}{2}}
\le C \,(1+j\delta)^{d+1}\, e^{-\frac{(j\delta)^2}{2}}.
\end{align}
Now, we define
\begin{align}
S(\delta) := \sum_{j=2}^\infty (1+j\delta)^{d+1}\, e^{-\frac{(j\delta)^2}{2}}.
\end{align}
Note that we we can upper bound as follows
\begin{align}
\sum_{j=2}^\infty \Bigl| H_{d+1}\bigl(j\delta\bigr) \Bigr|\, e^{-\frac{(j\delta)^2}{2}}
\le C\, S(\delta).
\end{align}
For each fixed \( j \ge 2 \), notice that $(1+j\delta)^{d+1}\, e^{-\frac{(j\delta)^2}{2}}$
decays exponentially in \( j \) (since the exponential term \( e^{-\frac{(j\delta)^2}{2}} \) dominates the polynomial growth of \((1+j\delta)^{d+1}\)). Moreover, for fixed \( j\ge2 \) we have
\begin{align}
\lim_{\delta\to\infty} (1+j\delta)^{d+1}\, e^{-\frac{(j\delta)^2}{2}} = 0.
\end{align}
Thus, the series \( S(\delta) \) converges for every fixed \(\delta > 0\) and
\begin{align}
\lim_{\delta\to\infty} S(\delta) = 0.
\end{align}
Hence, by the definition, there exists some \(\delta_0 > 0\) such that for all \(\delta \ge \delta_0\) we have
\begin{align}
S(\delta) < \frac{\rho}{4C}.
\end{align}
It follows that for every \(\delta \ge \delta_0\),
\begin{align}
\sum_{j=2}^\infty \Bigl| H_{d+1}\bigl(j\delta\bigr) \Bigr|\, e^{-\frac{(j\delta)^2}{2}}
\le C\cdot S(\delta) < C\cdot \frac{\rho}{4C} = \frac{\rho}{4}.
\end{align}

Thus for this choice of $\delta_0$ we achieve the statement of the lemma.
\end{proof}

%% file: app_example.tex
\section{A Sequence With Diverging $\ell_1$-Norm and Converging RKHS Norm}\label{app: diverging}
In this Appendix, we provide the proof of \exmref{exam: divergence}.

\begin{lemma}
Let $\alpha_n = \frac{1}{n}$ and suppose that the points $\bm{x}_n \in \mathbb{R}^d$ satisfy 
\begin{align}
\|\bm{x}_i-\bm{x}_j\| \ge |i-j|\delta \quad \text{for some } \delta>0 \text{ and for all } i,j \in \mathbb{N}.
\end{align}
Then the function
\begin{align}
f(\bm{x})=\sum_{n=1}^\infty \frac{1}{n}\, k(\bm{x},\bm{x}_n),
\end{align}
with the Gaussian kernel 
\begin{align}
k(\bm{x},\bm{y})=\exp\Bigl(-\frac{\|\bm{x}-\bm{y}\|^2}{2\sigma^2}\Bigr),
\end{align}
has finite RKHS norm
\begin{align}
\|f\|_{\mathcal{H}}^2 = \sum_{i,j=1}^\infty \frac{1}{ij}\, k(\bm{x}_i,\bm{x}_j) < \infty,
\end{align}
even though
\begin{align}
\|\alpha\|_{\ell_1} = \sum_{n=1}^\infty \frac{1}{n} = \infty.
\end{align}
\end{lemma}

\begin{proof}
We begin by splitting the double series defining the RKHS norm into diagonal and off--diagonal parts:
\begin{align}
\|f\|_{\mathcal{H}}^2 = \sum_{i=1}^\infty \frac{1}{i^2}\, k(\bm{x}_i,\bm{x}_i)
+\sum_{i\neq j} \frac{1}{ij}\, k(\bm{x}_i,\bm{x}_j).
\end{align}
Since $k(\bm{x},\bm{x})=1$ for all $\bm{x}\in\mathbb{R}^d$, the diagonal contribution is
\begin{align}
S_{\text{diag}} = \sum_{i=1}^\infty \frac{1}{i^2},
\end{align}
which converges (indeed, $\sum_{i=1}^\infty \frac{1}{i^2}=\pi^2/6$).

For the off--diagonal part, define
\begin{align}
S_{\text{off}} = \sum_{i\neq j} \frac{1}{ij}\, k(\bm{x}_i,\bm{x}_j).
\end{align}
By symmetry and non-negativity of $k(\bm{x}_i,\bm{x}_j)$, we can write
\begin{align}
S_{\text{off}} = 2\sum_{i>j} \frac{1}{ij}\, k(\bm{x}_i,\bm{x}_j).
\end{align}
For $i>j$, the separation condition implies
\begin{align}
\|\bm{x}_i-\bm{x}_j\| \ge (i-j)\delta,
\end{align}
so that
\begin{align}
k(\bm{x}_i,\bm{x}_j) = \exp\Bigl(-\frac{\|\bm{x}_i-\bm{x}_j\|^2}{2\sigma^2}\Bigr)
\le \exp\Bigl(-\frac{((i-j)\delta)^2}{2\sigma^2}\Bigr).
\end{align}
Setting 
\begin{align}
k = i - j \quad (k\ge 1)
\end{align}
and writing $i=j+k$, we obtain
\begin{align}
S_{\text{off}} \le 2\sum_{k=1}^\infty \exp\Bigl(-\frac{(k\delta)^2}{2\sigma^2}\Bigr)
\sum_{j=1}^\infty \frac{1}{j(j+k)}.
\end{align}
We now analyze the inner sum. Using partial fractions,
\begin{align}
\frac{1}{j(j+k)} = \frac{1}{k}\Bigl(\frac{1}{j} - \frac{1}{j+k}\Bigr).
\end{align}
Thus,
\begin{align}
\sum_{j=1}^\infty \frac{1}{j(j+k)} = \frac{1}{k} \sum_{j=1}^\infty \left(\frac{1}{j} - \frac{1}{j+k}\right).
\end{align}
The telescoping sum yields
\begin{align}
H_k := \sum_{j=1}^\infty \left(\frac{1}{j} - \frac{1}{j+k}\right)
=\sum_{j=1}^{k} \frac{1}{j},
\end{align}
where $H_k$ is also known as the $k$th harmonic number. Hence,
\begin{align}
\sum_{j=1}^\infty \frac{1}{j(j+k)} = \frac{H_k}{k}.
\end{align}
It follows that
\begin{align}
S_{\text{off}} \le 2\sum_{k=1}^\infty \exp\Bigl(-\frac{(k\delta)^2}{2\sigma^2}\Bigr) \frac{H_k}{k}.
\end{align}
For large $k$, it holds that
\begin{align}
H_k = \ln k + \gamma + o(1),
\end{align}
where $\gamma$ is the Euler--Mascheroni constant. Moreover, the Gaussian factor 
\begin{align}
\exp\Bigl(-\frac{(k\delta)^2}{2\sigma^2}\Bigr)
\end{align}
decays exponentially in $k$. Therefore, the series
\begin{align}
\sum_{k=1}^\infty \exp\Bigl(-\frac{(k\delta)^2}{2\sigma^2}\Bigr) \frac{H_k}{k}
\end{align}
is bounded, where we note that $\tfrac{H_k}{k}$ is bounded from above by $1$ (for all $k$ taken sufficiently large). 
Combining the diagonal and off--diagonal parts, we deduce that

\begin{align}
    \|f\|_{\mathcal{H}}^2 & =  \sum_{i=1}^\infty \frac{1}{i^2}\, k(\bm{x}_i,\bm{x}_i)
+\sum_{i\neq j} \frac{1}{ij}\, k(\bm{x}_i,\bm{x}_j) \\
& = S_{\text{diag}} + S_{\text{off}} \\
& \le \sum_{i=1}^\infty \frac{1}{i^2} + 2\sum_{k=1}^\infty \exp\Bigl(-\frac{(k\delta)^2}{2\sigma^2}\Bigr) \frac{H_k}{k} \\
& < \infty.
\end{align}
\end{proof}